%% file: 0main.tex
\documentclass{article}

\usepackage[preprint]{neurips_2026}

\usepackage[utf8]{inputenc} 
\usepackage[T1]{fontenc}    
\usepackage{hyperref}       
\usepackage{url}            
\usepackage{booktabs}       
\usepackage{amsfonts}       
\usepackage{nicefrac}       
\usepackage{microtype}      
\usepackage{xcolor}         
\usepackage{multirow}
\usepackage{xspace}
\usepackage{amsmath}
\usepackage[framemethod=TikZ]{mdframed}
\usepackage{enumitem}
\usepackage{graphicx}
\usepackage{booktabs}
\usepackage[ruled,vlined]{algorithm2e}
\usepackage[accsupp]{axessibility}
\usepackage[most]{tcolorbox}
\usepackage{fvextra}
\usepackage{wrapfig}
\usepackage{booktabs}
\usepackage{amsthm}
\usepackage[utf8]{inputenc}
\usepackage{amssymb} 
\newtheorem{proposition}{Proposition}

\newcommand{\ours}{\texttt{MARGO}\xspace}

\title{Mitigating Factual Hallucination in Large Reasoning Models via Mixed-Mode Advantage Regularization}

\author{%
  Kaishen Wang, Tong Zheng, Xuehao Cui, Ruibo Chen, Tianyi Xiong, Heng Huang \\
  Department of Computer Science, University of Maryland, College Park\\
  \texttt{kaishen@umd.edu} \\
}

\begin{document}

\maketitle

\begin{abstract}
Large reasoning models (LRMs) improve language model capabilities by generating explicit thinking traces before final answers. In factuality-oriented question answering (QA), such thinking often improves overall performance by helping the model recover relevant knowledge and refine its answers. However, we find that this benefit is not uniform at the instance level: explicit thinking can also overturn correct non-thinking answers and lead to factual drift. We refer to this failure mode as \emph{thinking-induced hallucination}. To explain this phenomenon, we formulate explicit thinking in factuality QA as a thinking residual over the model's direct-answer tendency, which can either recover missing knowledge or introduce unsupported associations. Based on this formulation, we propose \ours, \underline{\textit{M}}ixed-Mode \underline{\textit{A}}dvantage \underline{\textit{R}}egularization for \underline{\textit{G}}rounded \underline{\textit{O}}ptimization, a reinforcement learning framework that uses non-thinking rollouts as same-model references in advantage estimation. By constructing mixed-mode rollout groups with both thinking and non-thinking trajectories, \ours evaluates whether explicit thinking adds factual value beyond direct answering, thereby suppressing hallucination-prone thinking while preserving beneficial thinking behaviors. Experiments across multiple factuality-oriented QA benchmarks demonstrate that \ours improves factual reliability over strong baselines, while evaluations on mathematical benchmarks show that it preserves general reasoning ability.
\end{abstract}

\input{1intro}
\input{2related}
\input{3method}
\input{4exp}

\input{5con}

\bibliographystyle{plainnat}
\bibliography{reference}

\newpage
\input{appendix}


\end{document}

%% file: 1intro.tex
\section{Introduction}

Large reasoning models (LRMs) have recently achieved strong performance by generating explicit thinking traces before final answers~\citep{jaech2024openai,guo2025deepseek,yang2025qwen3,hou2025t1}. This paradigm is especially effective for mathematics, coding, and symbolic reasoning, where intermediate steps help decompose problems, verify partial results, and derive final answers~\citep{wei2022chain,wang2022self,zheng2023progressive,gao2023pal}. Explicit thinking can also benefit factuality-oriented question answering (QA), as reasoning traces may help models recall relevant knowledge, compare candidate facts, and refine final responses. As a result, thinking has become a widely used strategy for improving modern large language models (LLMs).

However, in factuality QA, the benefit of thinking is not uniformly positive at the instance level. To examine how thinking changes model predictions, we compare non-thinking and thinking modes on TriviaQA~\citep{joshi2017triviaqa}. As shown in Table~\ref{tab:preliminary-experiments}, on Qwen3-8B~\citep{yang2025qwen3}, thinking corrects 11.74\% of examples that are originally incorrect under non-thinking, confirming its factual gains. At the same time, it overturns 7.50\% of originally correct non-thinking answers into incorrect predictions. We call this failure mode \emph{thinking-induced hallucination}, where explicit thinking introduces misleading or unsupported content and changes a correct direct answer into a hallucinated one. This is particularly concerning because the correct answer is already accessible to the same model under the non-thinking mode, but is lost after generating an explicit thinking trace.

This instability stems from a key difference between factual QA and derivation-based reasoning. In mathematical reasoning, intermediate steps can often derive the answer from given premises. In factual QA, however, the answer usually depends on retrieving specific knowledge~\citep{petroni2019language,roberts2020much,lewis2020retrieval}. When the model is uncertain about the relevant fact, longer thinking does not create new evidence; instead, it may introduce unsupported associations, amplify uncertainty, and cause factual drift~\citep{lyu2023faithful,huang2025survey,su2025between,wang2025damo}. Thus, the central challenge is to preserve the factual gains of thinking while preventing reasoning traces from corrupting correct direct-answer tendencies. These observations motivate a residual view of explicit thinking in factuality QA. The non-thinking response reflects the model's direct factual tendency before an explicit reasoning trace is introduced, while the thinking response adds a reasoning residual on top of this tendency. This residual is helpful when it recovers the correct fact, but harmful when it introduces unsupported associations, entity confusion, or speculative intermediate claims. From this perspective, factuality optimization should account for how thinking changes the model's direct-answer tendency, rather than rewarding final answers in isolation.

\begin{figure}
    \centering
    \includegraphics[width=0.95\linewidth]{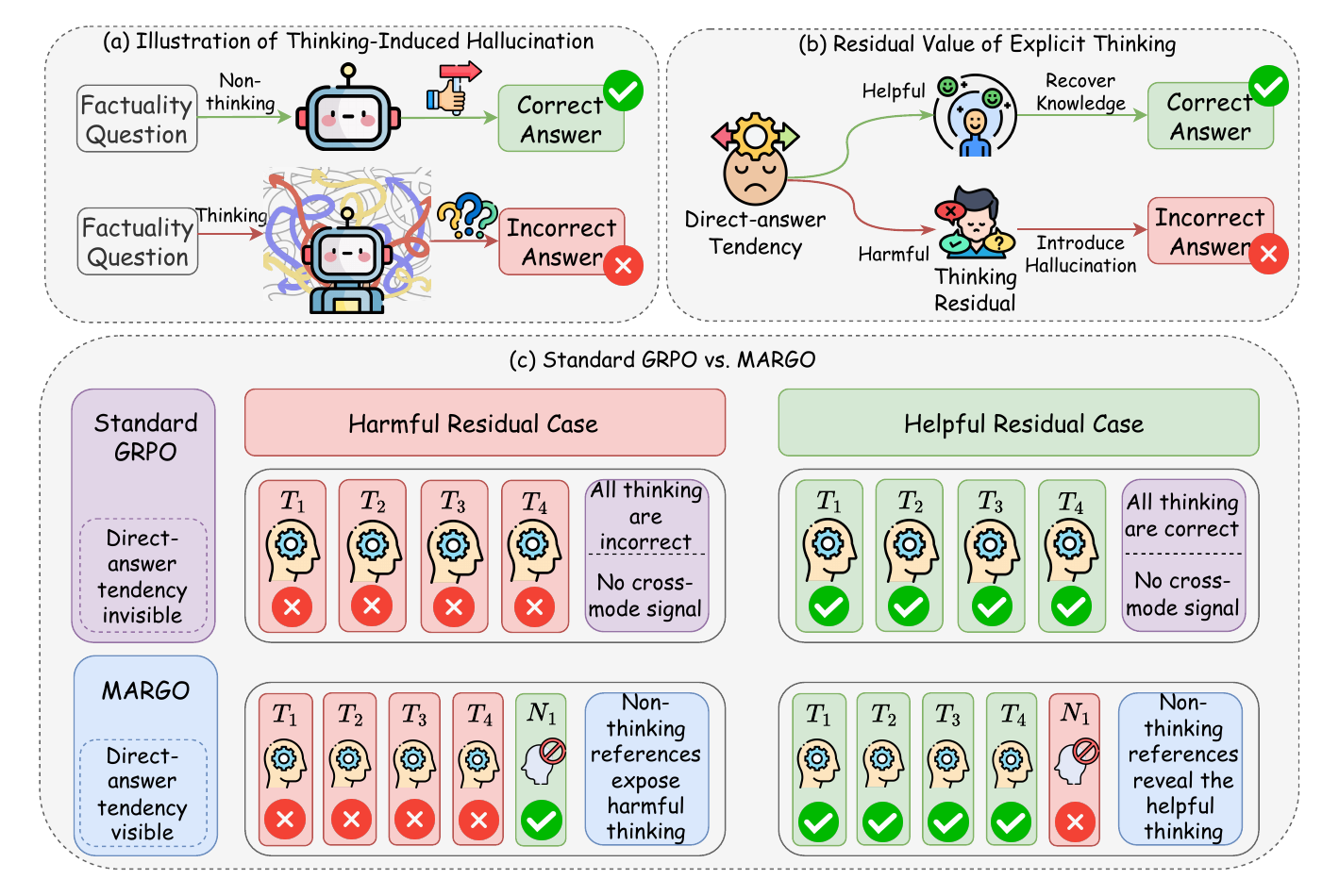}
    \vspace{-2mm}
    \caption{Overview of \ours. (a) Explicit thinking can induce hallucination by overturning a correct non-thinking answer into an incorrect prediction. (b) We view explicit thinking in LRMs as a residual over the model's direct-answer tendency, which can be either helpful or harmful. (c) Unlike standard GRPO, which compares only thinking rollouts, \ours uses non-thinking rollouts as same-model references to suppress harmful thinking while preserving helpful thinking.}
    \label{fig:overall_fig}
    \vspace{-7mm}
\end{figure}

Building on this interpretation, we propose \ours, \underline{\textit{M}}ixed-Mode \underline{\textit{A}}dvantage 
\underline{\textit{R}}egularization for \underline{\textit{G}}rounded \underline{\textit{O}}ptimization, 
a reinforcement learning framework for mitigating thinking-induced hallucination in factuality QA. Unlike standard GRPO with all-thinking rollout groups, \ours constructs mixed rollout groups containing both thinking and non-thinking trajectories for each question. This design makes the non-thinking response a same-model reference in advantage estimation, so thinking trajectories are optimized relative to the model's direct-answer behavior rather than only against other thinking samples. As a result, hallucination-prone reasoning that corrupts a correct non-thinking answer receives lower relative advantage, while reasoning that improves factual correctness remains reinforced. The overall architecture is shown in Figure~\ref{fig:overall_fig}.

We train \ours and evaluate it on two models across six factuality-oriented QA benchmarks. Experimental results show that \ours consistently improves factual reliability over fixed-mode inference and training-based baselines, including mode selection, supervised fine-tuning, and RL variants trained with the same curated data. Evaluation on mathematical benchmarks further shows that \ours preserves the general reasoning ability of LRMs. The contributions of this paper are summarized as follows:

\begin{itemize}
    \vspace{-1mm}
    \item We identify the \emph{thinking-induced hallucination} issue in factuality-oriented QA, where explicit thinking turns correct non-thinking answers into incorrect ones, and show through instance-level transition analysis that this phenomenon appears consistently across different LRM scales.
    \vspace{-0.75mm}
    \item We formulate explicit thinking in factuality QA as a reasoning residual over the model's direct-answer tendency, and propose \ours, a mixed-mode advantage regularization framework that uses non-thinking rollouts as same-model references to suppress hallucination-prone thinking while preserving trajectories that provide factual value.
\end{itemize}

%% file: 2related.tex
\vspace{-2.5mm}

\section{Related Work}

\vspace{-0.5mm}
\subsection{Hallucination in Large Reasoning Models}
\vspace{-0.5mm}
Although large reasoning models (LRMs) have achieved strong performance on reasoning-intensive tasks, recent studies suggest that explicit reasoning may introduce additional factuality risks. In factuality-oriented settings, long reasoning traces can introduce unsupported intermediate claims, amplify early mistakes, or produce plausible but factually incorrect conclusions~\citep{yao2025reasoning,li2025hallucination,chen2025learning,yu2025hallurnn,chen2025improving,wang2026unsafe}. 
For example, \citep{chen2025learning} show that LRMs generate substantially more hallucinations on long-form factuality benchmarks and propose online RL with rewards balancing factual precision, response detail, and relevance. \citep{li2025hallucination} further analyze hallucination caused by reasoning-oriented RL and propose factuality-aware step-wise policy optimization. Another line of work detects or characterizes hallucinations inside reasoning traces: \citep{sun2025detection} study reasoning hallucination from a mechanistic perspective, where logically coherent but factually incorrect traces lead to persuasive yet faulty conclusions, while \citep{zhangunraveling} analyze LRM hallucinations by representing reasoning trajectories as structured graphs. Different from these studies, we focus on an instance-level mode-induced failure: the same model can answer correctly without thinking, but becomes incorrect after generating an explicit thinking trace.

\vspace{-0.5mm}
\subsection{Reasoning Optimization and Adaptive Thinking}
\vspace{-0.5mm}
Reinforcement learning has become a common approach for improving the reasoning ability of LLMs, with recent methods showing that RL can elicit long-form reasoning behaviors such as exploration, reflection, and self-correction~\citep{ouyang2022training,shao2024deepseekmath,guo2025deepseek}. 
GRPO~\citep{shao2024deepseekmath} further provides a practical framework for reasoning-oriented training by estimating relative advantages within groups of sampled responses~\citep{shao2024deepseekmath}. 
Meanwhile, adaptive thinking methods study how to decide whether, when, or how long a model should reason, often treating thinking as a mode-selection or budget-allocation problem~\citep{wan2025adapthink,zhang2025adaptthink,han2025your,liadaptive,jali2026not,xiong2026multi,zheng2026parallel}. 
Different from these approaches, \ours does not learn a question-level policy for when to think, nor does it optimize only all-thinking rollout groups. 
Instead, it constructs mixed rollout groups containing both thinking and non-thinking trajectories, using non-thinking responses as same-model references to regularize reasoning traces that degrade factuality while preserving those that provide factual value.

\vspace{-1mm}

%% file: 3method.tex
\section{Method}
\vspace{-1mm}
\subsection{Preliminary: Thinking and Non-Thinking Inference Modes}
\label{sec:preliminary}

Recent large reasoning models (LRMs) are typically deployed under two distinct inference modes, namely \emph{thinking} and \emph{non-thinking}. These modes are not selected adaptively by the model itself at inference time; instead, they are determined by the prompting template.

Given an input question $x$, let $o=(o_1,\dots,o_T)$ denote the output sequence generated by an LRM parameterized by $\theta$. The generation process follows the standard autoregressive factorization:
\begin{equation}
P_\theta(o \mid x) = \prod_{t=1}^{T} P_\theta(o_t \mid x, o_{<t}).
\end{equation}
The difference between the two modes lies in the prompting format and the resulting structure of the generated response. In the thinking mode, the model is prompted to generate an explicit thinking trace before producing the final answer. Denoting the reasoning trace by $r$ and the final answer by $y$, the generated response takes the form:
\begin{equation}
o = [\texttt{<think>},\, r,\, \texttt{</think>},\, y],
\end{equation}
where $r$ is a non-empty reasoning segment enclosed by the \texttt{<think>} and \texttt{</think>} tags. In the non-thinking mode, the prompting template already provides an empty reasoning block after the question, i.e., \texttt{<think></think>} is pre-filled as part of the input context. As a result, the model directly generates the final answer $y$ without producing an explicit thinking trace. 

Therefore, the distinction between thinking and non-thinking does not stem from any change in model parameters, but from the prompting template and the induced generation pattern. Although both modes operate on the same underlying model, they may lead to substantially different behaviors.

\subsection{Empirical Observation: Thinking-Induced Hallucination}
\label{sec:motivation}

While LRMs often benefit from additional test-time computation, whether explicit thinking is uniformly beneficial for factuality tasks remains unclear. 
To investigate this, we conduct preliminary experiments comparing the thinking and non-thinking modes on TriviaQA~\citep{joshi2017triviaqa}, a representative factuality benchmark. 
Beyond overall performance, we further perform a fine-grained comparison at the instance level to examine how prediction correctness changes when switching from the non-thinking mode to the thinking mode.

\begin{wraptable}{l}{0.48\linewidth}
\centering
\small
\setlength{\tabcolsep}{6pt}
\begin{tabular}{ccccc}
\hline

\hline
Size & $r_{0,0}$ & $r_{0,1}$ & $r_{1,0}$ & $r_{1,1}$ \\
\hline
1.7B & 62.98\% & 12.32\% & 5.38\% & 19.32\% \\
4B   & 48.06\% & 12.15\% & 5.83\% & 33.96\% \\
8B   & 36.08\% & 11.74\% & 7.50\% & 44.68\% \\
14B & 30.43\% & 10.45\% & 6.75\% & 52.36\% \\
32B  & 28.60\% & 10.22\% & 6.27\% & 54.92\% \\
\hline

\hline
\end{tabular}
\caption{Transition ratios between non-thinking and thinking modes on TriviaQA across Qwen3 models of different scales. Each $r_{i,j}$ denotes the fraction of examples whose prediction changes from correctness $i$ under the non-thinking mode to correctness $j$ under the thinking mode.}
\label{tab:preliminary-experiments}
\end{wraptable}

Formally, for each example $x$, let $c^{\text{N}}(x)\in\{0,1\}$ and $c^{\text{T}}(x)\in\{0,1\}$ denote the correctness of the prediction under the non-thinking mode and the thinking mode, respectively, where $1$ indicates a correct prediction and $0$ otherwise. 
Based on the pair $\big(c^{\text{N}}(x), c^{\text{T}}(x)\big)$, we partition all examples into four transition groups:
\begin{equation}
\tau_{i,j}=\{x \mid c^{\text{N}}(x)=i,\ c^{\text{T}}(x)=j\}, \quad i,j\in\{0,1\}.
\end{equation}
Specifically, $\tau_{0,0}$ denotes the examples that are incorrect under both modes, $\tau_{0,1}$ denotes the examples that change from incorrect to correct, $\tau_{1,0}$ denotes the examples that change from correct to incorrect, and $\tau_{1,1}$ denotes the examples that remain correct under both modes. 
Among these four groups, $\tau_{1,0}$ is particularly important, as it captures the cases where thinking overturns an originally correct non-thinking prediction and induces hallucination. 
We further define the transition ratio of each group as:
\begin{equation}
r_{i,j} = \frac{|\tau_{i,j}|}{|\mathcal{D}|}, \quad i,j\in\{0,1\},
\end{equation}
where $\mathcal{D}$ denotes the evaluation set.
As shown in Table~\ref{tab:preliminary-experiments}, although thinking can correct a portion of non-thinking incorrect predictions (i.e., $r_{0,1}$), there also exists a non-negligible fraction of examples in $r_{1,0}$, where thinking overturns an originally correct non-thinking answer and leads to an incorrect prediction. 
We refer to this failure mode as \emph{thinking-induced hallucination}. 
These results reveal that explicit thinking plays a dual role in factuality QA: it can recover some missed answers, but it can also introduce harmful reasoning that drifts away from an originally correct direct answer. 
We provide additional transition analyses in Appendix~\ref{app:additional-transition}.

We further inspect the thinking traces in $\tau_{1,0}$ and find that they often contain uncertainty or revision markers, such as ``But wait'', ``Hmm, not sure'', ``Wait, no, maybe that's not right'', and ``but I'm getting confused''. 
These traces suggest that thinking-induced hallucination is often not caused by a complete absence of the correct fact, but by instability introduced during the thinking process. 
When the relevant factual knowledge is uncertain, longer thinking does not create new evidence; instead, it may amplify uncertainty, introduce entity confusion, or over-elaborate unsupported associations. 
Since the two modes share the same model parameters and differ mainly by whether an explicit thinking trace is generated, the non-thinking response reveals the model's direct factual tendency before this additional thinking process is introduced. 
This motivates our residual view of explicit thinking: in factuality QA, thinking can be seen as a reasoning residual on top of the model's direct-answer behavior, and the value of this residual should be evaluated relative to the non-thinking reference.

\subsection{Residual Value of Explicit Thinking}
\label{sec:residual_value}

The transition analysis above motivates modeling thinking as a \emph{reasoning residual} over the model's non-thinking behavior. For the same question $x$, the non-thinking response reflects the model's direct factual tendency, while the thinking response adds an explicit thinking trace before producing the final answer. 
The key question is therefore whether this additional reasoning residual provides positive factual value over the non-thinking reference.

Formally, let $\mu_m(x)$ denote the expected factual reward under mode $m\in\{\mathrm{T},\mathrm{N}\}$:
\begin{equation}
\mu_m(x)
=
\mathbb{E}_{y^{m}\sim \pi_\theta(\cdot|x,m)}
\left[
R(x,y^{m})
\right].
\end{equation}
We define the residual value of explicit thinking as:
\begin{equation}
\Delta_{\mathrm{res}}(x)
=
\mu_{\mathrm{T}}(x)-\mu_{\mathrm{N}}(x).
\end{equation}
A positive $\Delta_{\mathrm{res}}(x)$ means that thinking provides factual value beyond direct answering, while a negative value means that thinking introduces harmful residual contamination.

Standard GRPO with all-thinking rollout groups cannot directly capture this residual value. When all trajectories in a group are sampled in the thinking mode, the reward baseline is determined only by other thinking trajectories. As a result, a hallucination-prone thinking trajectory may still receive a favorable relative advantage as long as it performs better than other thinking samples in the same group. This optimizes relative quality within the thinking distribution, but does not evaluate whether thinking itself adds factual value over the model's direct-answer behavior.

To expose the residual value of thinking, we construct mixed rollout groups containing both thinking and non-thinking trajectories. Suppose a rollout group contains thinking and non-thinking trajectories with proportions $\alpha$ and $1-\alpha$, respectively. 
The expected mixed reward baseline is:
\begin{equation}
\mu_{\mathrm{mix}}(x)
=
\alpha\mu_{\mathrm{T}}(x)
+
(1-\alpha)\mu_{\mathrm{N}}(x).
\end{equation}
Ignoring the normalization scale for clarity, the expected relative advantage of a thinking trajectory becomes:
\begin{equation}
\begin{aligned}
A_{\mathrm{mix}}^{\mathrm{T}}(x)
&\propto
R(x,y^{\mathrm{T}})-\mu_{\mathrm{mix}}(x) \\
&=
\underbrace{R(x,y^{\mathrm{T}})-\mu_{\mathrm{T}}(x)}_{\text{within-mode advantage}}
+
(1-\alpha)
\underbrace{\Delta_{\mathrm{res}}(x)}_{\text{residual value}} .
\end{aligned}
\end{equation}
The first term measures the relative quality of a thinking trajectory within the thinking distribution, while the second term adjusts this advantage by the residual value of thinking over the non-thinking reference. Therefore, when $\Delta_{\mathrm{res}}(x)<0$, thinking is expected to introduce harmful residual contamination, and thinking trajectories receive lower relative advantage under the mixed baseline. When $\Delta_{\mathrm{res}}(x)>0$, thinking provides positive residual value and is reinforced. By symmetry, non-thinking trajectories receive the opposite residual adjustment, making them relatively preferred when thinking harms factuality. We present the detailed theoretical analysis in Appendix~\ref{appendix:theory}.


\subsection{Mixed-Mode Advantage Regularization for Grounded Optimization}
\label{sec:grpo_optimization}

Following the residual-value formulation above, we instantiate \ours with GRPO by constructing mixed rollout groups for each factual question. The key idea is to place thinking and non-thinking trajectories into the same group, so that the GRPO advantage reflects not only within-mode reward differences but also whether thinking adds factual value over the model's direct-answer behavior.

Given an input question $x$, we sample both thinking and non-thinking responses:
\begin{equation}
\mathcal{G}(x)
=
\{y^{\mathrm{T}}_1,\ldots,y^{\mathrm{T}}_{K_T}\}
\cup
\{y^{\mathrm{N}}_1,\ldots,y^{\mathrm{N}}_{K_N}\},
\end{equation}
where $y^{\mathrm{T}}_i$ is sampled under the thinking prompt and $y^{\mathrm{N}}_j$ is sampled under the non-thinking prompt. Both modes share the same model parameters and differ only in the prompting template. Thus, the mixed rollout group provides both thinking trajectories and same-model non-thinking references for the same question.

For each response $y_i \in \mathcal{G}(x)$, we compute a reward $r_i$ based on factual correctness and response format. GRPO then normalizes rewards within the mixed group:
\begin{equation}
A_i =
\frac{r_i - \mathrm{mean}(\{r_j\}_{j=1}^{|\mathcal{G}(x)|})}
{\mathrm{std}(\{r_j\}_{j=1}^{|\mathcal{G}(x)|})},
\qquad
\rho_i(\theta)=
\frac{\pi_\theta(y_i\mid x,m_i)}
{\pi_{\theta_{\mathrm{old}}}(y_i\mid x,m_i)},
\end{equation}
where $m_i\in\{\mathrm{T},\mathrm{N}\}$ denotes the prompting mode used to sample $y_i$. This mixed-group normalization implements the residual-value principle in Section~\ref{sec:residual_value}: thinking trajectories are evaluated not only against other thinking samples, but also against non-thinking references for the same question.

The clipped GRPO objective is defined as:
\begin{equation}
\mathcal{J}_{\mathrm{GRPO}}(\theta)
=
\mathbb{E}_{x,\mathcal{G}(x)}
\left[
\frac{1}{|\mathcal{G}(x)|}
\sum_{y_i\in \mathcal{G}(x)}
\min\Big(
\rho_i(\theta)A_i,\,
\mathrm{clip}\big(\rho_i(\theta),1-\epsilon,1+\epsilon\big)A_i
\Big)
\right].
\end{equation}
In practice, we additionally apply KL regularization to stabilize training and constrain the updated policy from drifting too far from the reference model.

\paragraph{Reward for Factuality QA.}
Unlike conventional rule-based RLVR tasks such as mathematical reasoning, factuality QA does not admit a simple deterministic verifier. In math problems, the final answer can often be extracted and matched against the ground truth. In contrast, factual answers may appear in different surface forms, aliases, abbreviations, or paraphrases, making exact string matching unreliable. Therefore, we use Qwen3-32B as an automatic judge to evaluate sampled rollouts, consistent with~\citep{yao2025reasoning}. The judge is not asked to verify factual correctness from its own parametric knowledge. Instead, it is provided with the question, the ground-truth answer, and the model prediction, and is required to determine whether the prediction is semantically consistent with the provided ground truth. The detailed judging prompt is provided in Appendix~\ref{appendix:prompt_for_evaluation}.

Based on the judge result, we define the reward as:
\begin{equation}
r = r_{\mathrm{corr}} + r_{\mathrm{fmt}}.
\end{equation}
Here, $r_{\mathrm{corr}}$ measures factual correctness, while $r_{\mathrm{fmt}}$ provides a lightweight bonus for valid response formatting. The correctness reward is defined as:
\begin{equation}
r_{\mathrm{corr}} =
\begin{cases}
+1, & \text{if the response is judged correct},\\
-1, & \text{otherwise}.
\end{cases}
\end{equation}
A response is regarded as format-valid if it contains a properly closed \texttt{<think>} and \texttt{</think>}  blocks. We define the format reward as:
\begin{equation}
r_{\mathrm{fmt}} =
\begin{cases}
+0.05, & \text{if the response format is valid},\\
0, & \text{otherwise}.
\end{cases}
\end{equation}

Since GRPO computes relative advantages over both thinking and non-thinking trajectories for the same question, a hallucination-prone thinking trajectory receives lower advantage when the non-thinking reference achieves a higher reward. This makes the regularization effect emerge from group-wise comparison rather than from manually designed mode rewards.

%% file: 4exp.tex
\section{Experiments}

\subsection{Experimental Setting}\label{sec:experimental-setting}
\vspace{-1mm}
\paragraph{Models and Benchmarks.}
We evaluate \ours on two LRMs, Qwen3-4B and Qwen3-8B~\citep{yang2025qwen3}. For evaluation, we use six representative factuality-oriented benchmarks: SimpleQA~\citep{wei2024measuring}, SimpleQA-Verified~\citep{haas2025simpleqa} (denoted as SimpleQA-V in our paper), TriviaQA~\citep{joshi2017triviaqa}, NQ$\_$Open~\citep{kamalloo2023evaluating}, PopQA~\citep{mallen2023not}, and HotpotQA~\citep{yang2018hotpotqa}. We use Qwen3-32B as an automatic judge to determine whether the model prediction is semantically consistent with the ground-truth answer~\citep{yao2025reasoning}. Detailed descriptions and the judging prompt are provided in Appendix~\ref{appendix:prompt_for_evaluation} and Appendix~\ref{appendix:models-and-benchmarks}.

\vspace{-2.2mm}
\paragraph{Training Data Construction.} 
We construct model-specific training data from the TriviaQA training split. For each base model, we run both thinking and non-thinking modes on the training questions and retain examples where the two modes show a clear factuality gap, including cases where thinking performs much better and cases where non-thinking performs much better. This filtering is performed separately for Qwen3-4B and Qwen3-8B, since the effect of thinking is model-dependent. Details are provided in Appendix~\ref{appendix:data_construction}.

 \vspace{-2.2mm}
\paragraph{Baselines.}
We compare \ours with several baselines trained  under the same data setting. Fixed$_{{NoThink}}$ and Fixed$_{{Think}}$ denote the original model evaluated under the non-thinking and thinking modes, respectively. Adaptive$_{{CLS}}$ trains \texttt{microsoft/deberta-v3-base}~\citep{he2021debertav3} as a binary classifier to predict whether a given question should invoke thinking. We also include Adaptive$_{{SFT}}$, which uses the same data to teach the model when to think and when not to think via SFT. Adaptive$_{{RL}}$~\citep{zhang2025adaptthink} also allows the model to select between thinking and non-thinking modes via reinforcement learning. In addition, to isolate the effect of mixed-mode advantage regularization from reinforcement learning alone, we include two fixed-mode RL baselines. Fixed$_{{NoThink+RL}}$ fine-tunes the model with GRPO using the non-thinking format for all training examples, while Fixed$_{{Think+RL}}$ fine-tunes the model with GRPO using the thinking format for all training examples. All training-based baselines use the same curated training data as \ours. The detailed descriptions are provided in Appendix~\ref{appendix:baselines}.

\paragraph{Training and Evaluation Details.}
We train \ours using the \texttt{verl}~\citep{sheng2025hybridflow} framework. All experiments are conducted on 8 NVIDIA H200 GPUs, where 4 GPUs are used for policy training and the remaining 4 GPUs are used to serve the reward model. We train the model for one epoch with a training batch size of 64, a learning rate of $1\times10^{-6}$, and a maximum response length of 8192 tokens. 
For each prompt, we sample 8 rollouts, including 6 thinking rollouts and 2 non-thinking rollouts; thus, the thinking-ratio coefficient $\alpha$ is set to 0.75. We also apply KL regularization to constrain the updated policy from drifting too far from the reference model, with the KL loss coefficient set to 0.001. We evaluate all benchmarks using \texttt{vLLM}~\citep{kwon2023efficient}. 
For factuality evaluation, we use the same judge model, Qwen3-32B, and the same prompt strategy as in training. Following the official decoding recommendations, we set the temperature to 0.6, top-$p$ to 0.95, and top-$k$ to 20 for thinking mode generation. 
For non-thinking mode generation, we set the temperature to 0.7, top-$p$ to 0.8, and top-$k$ to 20. 
The maximum generation length is set to 8192 tokens.


\subsection{Experimental Results}

\begin{table}[t]
  \centering
    \resizebox{\textwidth}{!}{
    \begin{tabular}{cccccccc}
    \hline
    
    \hline
       Method & SimpleQA & SimpleQA-V & TriviaQA & NQ$\_$Open & PopQA & HotpotQA & Average \\
    \hline
     & \multicolumn{6}{c}{\textit{Qwen3-4B}} \\

     Fixed$_{NoThink}$ & 3.74\% & 3.80\% & 39.79\% & 29.50\% & 17.07\% & 27.86\% & 20.29\%  \\

       Fixed$_{Think}$ & 3.98\% & 4.10\% & 46.12\% & 32.94\% & 19.05\% & 30.64\% & 22.81\%  \\

       Adaptive$_{CLS}$ & 3.61\% & 4.30\% & 43.91\% & 30.33\% & 18.36\% & 28.43\% & 21.49\% \\
       
       Adaptive$_{SFT}$ & 3.35\% & 4.80\% & 41.58\% & 30.28\% & 17.69\% & 28.08\%  & 20.96\%  \\

       Adaptive$_{RL}$ & 4.23\% & 4.00\% & 47.22\% & 32.49\% & 19.22\% & 30.50\% & 22.94\% \\

       Fixed$_{NoThink+RL}$ & 4.62\%  & 5.20\% & 40.14\% & 30.61\% & 18.40\% & 27.70\% & 21.11\%  \\

      Fixed$_{Think + RL}$  & 4.48\% & 4.10\% & 47.44\% & 33.80\% & 19.16\% & 30.89\% & 23.31\% \\

      \ours & \textbf{6.73\%} & \textbf{10.00\%} & \textbf{49.12\%} & \textbf{34.76\%} & \textbf{20.48\%} & \textbf{31.82\%} & \textbf{25.49\%} \\
      \hline

    & \multicolumn{6}{c}{\textit{Qwen3-8B}} \\
      
     Fixed$_{NoThink}$ & 4.16\% & 4.10\% & 52.18\% & 38.64\%  & 22.13\% & 33.13\% & 25.72\% \\
        Fixed$_{Think}$  & 5.04\% & 4.60\% & 56.42\% & 39.94\% & 23.71\% &  35.84\% & 27.59\% \\

       Adaptive$_{CLS}$ & 4.62\% & 4.30\% & 54.94\% & 38.50\%  & 23.21\%  & 34.20\% &  26.63\% \\

     Adaptive$_{SFT}$ & 4.30\% & 4.60\% & 52.38\% & 38.50\% & 21.41\% & 33.88\% & 25.85\% \\

   Adaptive$_{RL}$ & 5.27\% & 5.50\% & 56.05\% & 40.08\% & 23.28\% & 35.64\% & 27.64\%  \\

      Fixed$_{NoThink+RL}$ & 4.23\% & 5.60\% & 53.80\% & 39.09\% & 23.22\% & 34.26\% & 26.70\% \\

       Fixed$_{Think + RL}$  & 5.29\% & 4.90\% & 56.88\% & 40.66\% & 23.85\% & 35.93\% & 27.92\%  \\

       \ours & \textbf{5.82\%}  & \textbf{9.00\%} & \textbf{58.03\%}  & \textbf{42.24\%} & \textbf{25.23\%} & \textbf{37.07\%} & \textbf{29.57\%} \\

    \hline

    \hline
    \end{tabular}}
    \caption{Main results on six factuality-oriented QA benchmarks using Qwen3-4B and Qwen3-8B. We compare \ours against fixed-mode inference, adaptive mode-selection baselines, supervised fine-tuning, and RL variants trained with the same data.}
    \vspace{-4mm}

\label{tab:qwen3-series}%
\end{table}%

\vspace{-1mm}
\paragraph{Main Results.}
Table~\ref{tab:qwen3-series} reports the results on six factuality-oriented QA benchmarks across Qwen3-4B and Qwen3-8B. Overall, \ours achieves the best average performance on both model scales. On Qwen3-4B, \ours improves the average accuracy from 22.81\% under the fixed thinking mode to 25.49\%, yielding a gain of 2.68\%. On Qwen3-8B, \ours improves the average accuracy from 27.59\% to 29.57\%, yielding a gain of 1.98\%. These results show that explicitly regularizing thinking behavior with non-thinking references improves factual reliability beyond simply enabling the thinking mode at inference time.

Compared with adaptive mode-selection baselines, \ours also shows clear advantages. Adaptive$_{{CLS}}$ and Adaptive$_{{SFT}}$ do not consistently outperform the fixed thinking mode, suggesting that a simple question-level mode-selection strategy may be insufficient in this setting. This is consistent with our motivation that the usefulness of thinking is model-specific and trajectory-dependent. In contrast, \ours avoids making a hard mode-selection decision and instead regularizes thinking through mixed-mode advantage estimation. We further compare \ours with RL baselines trained on the same data. Although Fixed$_{{Think+RL}}$ and Fixed$_{{NoThink+RL}}$ improve over some fixed-mode baselines, they remain consistently below \ours. For example, on Qwen3-4B, Fixed$_{{Think+RL}}$ achieves an average accuracy of 23.31\%, while \ours reaches 25.49\%. On Qwen3-8B, Fixed$_{{Think+RL}}$ reaches 27.92\%, compared with 29.57\% for \ours. This indicates that the improvement does not come from reinforcement learning alone, but from the mixed-mode advantage regularization that evaluates thinking trajectories against same-model non-thinking references.

Across individual benchmarks, \ours achieves the best performance on all six datasets for both Qwen3-4B and Qwen3-8B. The gains are especially pronounced on SimpleQA and SimpleQA-Verified, where factual precision is critical and unsupported thinking can easily introduce incorrect claims. These consistent improvements demonstrate that \ours effectively mitigates thinking-induced hallucination while preserving the useful factual benefits of explicit thinking.

\begin{figure}[!t]
    \centering
    \includegraphics[width=0.98\linewidth]{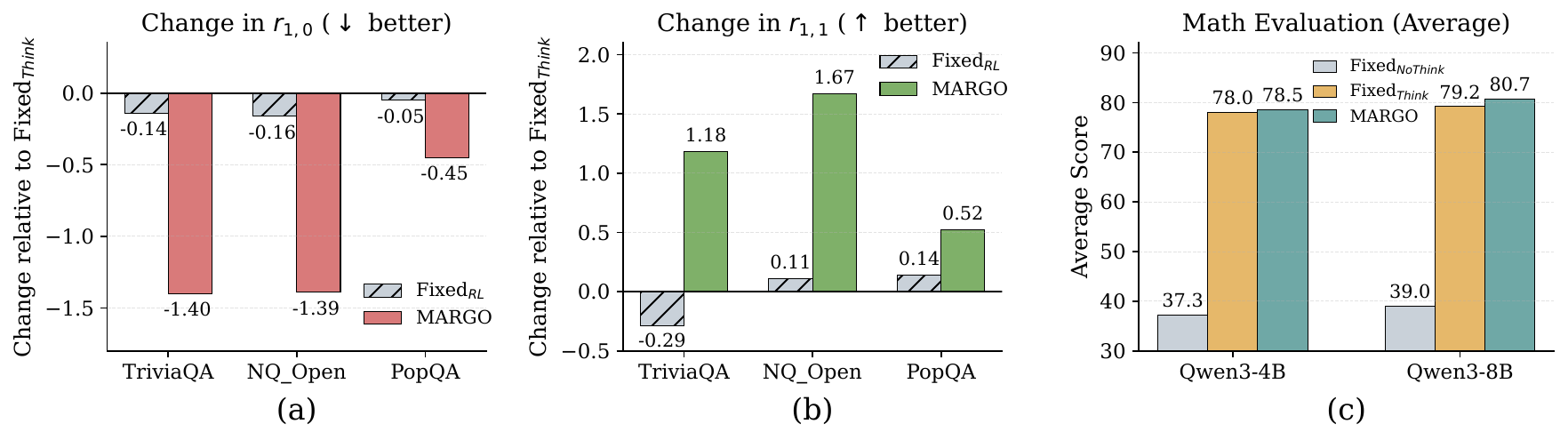}
    \vspace{-1mm}
    \caption{Analysis of transition-ratio changes and mathematical generalization. (a) Changes in $r_{1,0}$ on Qwen3-4B relative to Fixed$_{{Think}}$. (b) Changes in $r_{1,1}$ on Qwen3-4B relative to Fixed$_{{Think}}$. (c) Average mathematical performance over AMC23, AIME24, and AIME25 under Mean@16 for Qwen3-4B and Qwen3-8B, respectively.}
    \label{fig:transition_analysis}
    \vspace{-2mm}
\end{figure}

\paragraph{Transition Analysis.}
To further understand where the gains of \ours come from, we analyze how training changes the instance-level transition ratios on Qwen3-4B. We focus on $\tau_{1,0}$ and $\tau_{1,1}$, where $\tau_{1,0}$ captures cases in which thinking overturns an originally correct non-thinking answer, and $\tau_{1,1}$ captures cases that remain correct under thinking. Figure~\ref{fig:transition_analysis} reports the change of these two ratios relative to Fixed$_{{Think}}$. Fixed$_{{Think+RL}}$ produces only marginal changes in both $r_{1,0}$ and $r_{1,1}$, indicating that applying RL with the same data does not substantially alter the harmful transition pattern. 
In contrast, \ours consistently reduces $r_{1,0}$ across TriviaQA, NQ$\_$Open, and PopQA, while increasing $r_{1,1}$ on all three datasets. 
This shows that the improvement of \ours is not merely due to RL training or the curated data itself. 
Instead, the gains come from mixed-mode advantage regularization, where non-thinking rollouts serve as same-model references to suppress hallucination-prone thinking and preserve correct direct-answer tendencies. This also suggests that \ours does not simply bias the model toward non-thinking outputs, but instead changes the relative optimization signal according to whether thinking improves factual correctness for the same question.

\begin{figure}[htbp]
    \centering
    \includegraphics[width=0.99\linewidth]{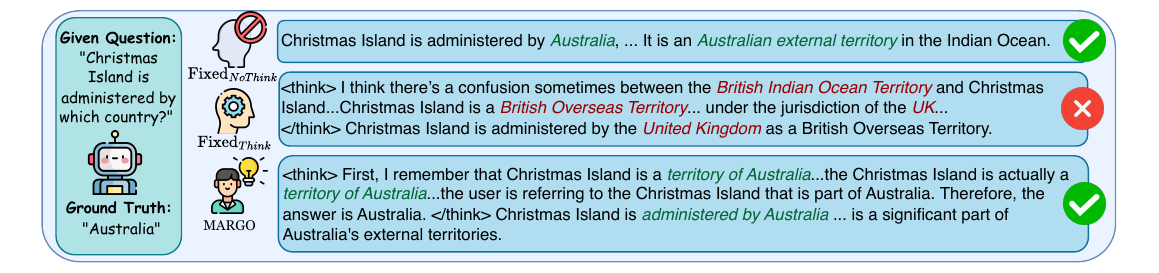}
    \caption{Qualitative case study comparing Fixed$_{{NoThink}}$, Fixed$_{{Think}}$, and \ours. 
    Fixed$_{{NoThink}}$ gives the correct direct answer, Fixed$_{{Think}}$ drifts to an incorrect answer through unsupported reasoning, and \ours recovers the correct answer by suppressing this factual drift.}
    \label{fig:case-study}
\end{figure}

\paragraph{Qualitative Analysis.}
We present a factuality QA case study to illustrate how \ours mitigates reasoning-induced hallucination. As shown in Figure~\ref{fig:case-study}, given the question ``Christmas Island is administered by which country?'', Fixed$_{{NoThink}}$ directly produces the correct answer, Australia, indicating that the model already contains the required factual knowledge in its direct-answer mode. However, when forced to generate an explicit reasoning trace, Fixed$_{{Think}}$ gradually drifts away from this correct knowledge. It introduces unsupported intermediate claims, confuses Christmas Island with British territories, and eventually predicts the United Kingdom as the administering country. This example demonstrates that additional reasoning can sometimes overwrite a correct direct answer with a hallucinated conclusion, rather than improving factual reliability. In contrast, \ours preserves the correct direct-answer tendency from the non-thinking rollout while still allowing a reasoning process. By using the non-thinking rollout as a same-model reference, \ours suppresses the erroneous factual drift in the thinking trajectory and recovers the correct answer.

\vspace{-1mm}

\subsection{Ablation Studies}
\paragraph{Mathematical Generalization.}
Although \ours is designed for mitigating hallucination in factuality-oriented QA, we further evaluate whether it preserves the general reasoning ability of LRMs on widely used mathematical benchmarks. Specifically, we compare the model trained with \ours against the original model under both non-thinking and thinking modes on three math benchmarks: AMC23, AIME24, and AIME25. For all three benchmarks, we report Mean@16 and Pass@16, where Mean@16 denotes the average accuracy over 16 sampled responses and Pass@16 denotes whether at least one of the 16 responses gives the correct answer. We set the maximum generation length to 32,768 for all math evaluations. As shown in Figure~\ref{fig:transition_analysis}(c), \ours preserves mathematical reasoning ability and even yields slight gains over the base model under the thinking mode. Specifically, \ours achieves the best Mean@16 across all three benchmarks at both the 4B and 8B scales, improving the average accuracy by 0.5\% on the 4B model and 1.5\% on the 8B model. The detailed results are provided in Table~\ref{tab:ablation-math}.

\begin{table}[!t]
  \centering
    \resizebox{\textwidth}{!}{
    \begin{tabular}{ccccccc}
    \hline
    
    \hline
       Method & SimpleQA & SimpleQA-V & TriviaQA & NQ$\_$Open & PopQA & HotpotQA \\
    \hline

     Fixed$_{NoThink}$ & 3.74\% & 3.80\% & 39.79\% & 29.50\% & 17.07\% & 27.86\%\\

       Fixed$_{Think}$ & 3.98\% & 4.10\% & 46.12\% & 32.94\% & 19.05\% & 30.64\% \\

    \ours + Random & 3.70\% & 4.50\% & 46.14\% & 32.88\% & 19.54\% & 30.34\%  \\
      
      \ours & \textbf{6.73\%} & \textbf{10.00\%} & \textbf{49.12\%} & \textbf{34.76\%} & \textbf{20.48\%} & \textbf{31.82\%} \\
      \hline

    \hline
    \end{tabular}}
    \caption{Data selection ablation on six factuality-oriented QA benchmarks using Qwen3-4B. The best results are highlighted in bold.}
    \vspace{-5mm}

\label{tab:ablation-data-4B}%
\end{table}%

\vspace{-2mm}
\paragraph{Data Selection Evaluation.}
Since \ours uses non-thinking rollouts as references to regularize thinking rollouts, the informativeness of the training prompts is important for effective optimization. Ideally, training examples should exhibit a clear factuality gap between thinking and non-thinking modes, so that the mixed rollout group can provide meaningful reference signals. To evaluate the effect of our data selection strategy, we conduct an ablation study on Qwen3-4B. We compare \ours trained on our selected training set with \ours trained on a random subset of 5,660 examples sampled from the same TriviaQA training pool. All training and evaluation settings are kept identical. As shown in Table~\ref{tab:ablation-data-4B}, random data selection yields limited improvement over the fixed thinking baseline, while our selected data leads to consistent gains across all benchmarks. This suggests that examples with clear thinking/non-thinking factuality gaps are more effective for mixed-mode advantage regularization.

\vspace{-2mm}

%% file: 5con.tex
\section{Conclusion}
\label{sec:conclusion}

In this paper, we studied factuality-oriented QA in large reasoning models (LRMs) and identified \emph{thinking-induced hallucination}, where explicit thinking can overturn correct non-thinking answers and lead to factual drift. We interpreted explicit thinking as a reasoning residual over the model's direct-answer behavior, which may either recover missing knowledge or introduce unsupported associations. Based on this view, we proposed \ours, a mixed-mode reinforcement learning framework that uses non-thinking rollouts as same-model references to regularize hallucination-prone thinking. Experiments across multiple factuality-oriented QA benchmarks show that \ours improves factual reliability over strong baselines, while preserving general reasoning ability on mathematical benchmarks. Overall, our findings suggest that improving LRMs requires not only encouraging more thinking, but also learning when explicit thinking provides factual value over direct answering.

%% file: appendix.tex
\section*{Limitations}
\label{sec:limitation}

This work focuses on improving the factual reliability of LRMs in factuality-oriented QA. While our experiments cover multiple benchmarks and model scales, future work can further explore the applicability of mixed-mode advantage regularization to broader scenarios, such as open-ended generation, multi-turn interaction, retrieval-augmented systems, and multimodal reasoning.

\section*{Broader Impact}\label{sec:broader-impact}
By reducing cases where explicit thinking turns correct direct answers into hallucinated predictions, \ours can help improve the reliability of LRMs in information-seeking and educational applications. At the same time, the method is intended to complement, rather than replace, external verification mechanisms in high-stakes domains. We hope this work encourages future research on more reliable and mode-aware optimization for LRMs.

\appendix

\section{Advantage Decomposition of Mixed-Mode Regularization}
\label{appendix:theory}

In this section, we provide a formal derivation showing why mixed-mode rollout groups introduce a residual-value adjustment that is absent in standard all-thinking GRPO. 
This analysis characterizes the advantage signal induced by group normalization, rather than providing a convergence guarantee.

\paragraph{Setup.}
For a factual question $x$, let $\mathrm{T}$ and $\mathrm{N}$ denote the thinking and non-thinking modes, respectively. 
A response sampled under mode $m\in\{\mathrm{T},\mathrm{N}\}$ is denoted by $y^m\sim \pi_\theta(\cdot\mid x,m)$, and its factual reward is $R(x,y^m)$. 
We define the expected reward of each mode as
\begin{equation}
\mu_m(x)
=
\mathbb{E}_{y^m\sim \pi_\theta(\cdot\mid x,m)}
\left[
R(x,y^m)
\right],
\qquad m\in\{\mathrm{T},\mathrm{N}\}.
\end{equation}
The residual value of explicit thinking is defined as
\begin{equation}
\Delta_{\mathrm{res}}(x)
=
\mu_{\mathrm{T}}(x)-\mu_{\mathrm{N}}(x).
\end{equation}
A positive $\Delta_{\mathrm{res}}(x)$ indicates that thinking provides higher expected factual reward than direct answering, while a negative value indicates that thinking harms factuality relative to the non-thinking reference.

\paragraph{All-Thinking GRPO.}
Consider a standard all-thinking rollout group where all samples are generated in the thinking mode. 
Ignoring the standard-deviation normalization for clarity, the group-relative advantage of a thinking response $y^{\mathrm{T}}$ is centered around the all-thinking reward baseline:
\begin{equation}
A_{\mathrm{T}}(x,y^{\mathrm{T}})
\propto
R(x,y^{\mathrm{T}})-\mu_{\mathrm{T}}(x).
\end{equation}
This advantage compares $y^{\mathrm{T}}$ only against other thinking-mode samples. 
Therefore, it can rank thinking trajectories within the thinking distribution, but it does not reveal whether thinking itself improves factuality over the model's direct non-thinking behavior.

\paragraph{Mixed-Mode GRPO.}
Now consider a mixed rollout group containing thinking and non-thinking trajectories with proportions $\alpha$ and $1-\alpha$, respectively. 
The expected group reward baseline is
\begin{equation}
\mu_{\mathrm{mix}}(x)
=
\alpha\mu_{\mathrm{T}}(x)
+
(1-\alpha)\mu_{\mathrm{N}}(x).
\end{equation}

\begin{proposition}[Residual-value decomposition of mixed-mode advantage]
\label{prop:residual_value_decomposition}
Ignoring the standard-deviation normalization and replacing the empirical group mean with its expectation, the mixed-mode advantage of a thinking trajectory can be decomposed as
\begin{equation}
A_{\mathrm{mix}}^{\mathrm{T}}(x,y^{\mathrm{T}})
\propto
\left(R(x,y^{\mathrm{T}})-\mu_{\mathrm{T}}(x)\right)
+
(1-\alpha)\Delta_{\mathrm{res}}(x).
\end{equation}
The first term is the within-thinking advantage, and the second term is the residual-value adjustment. 
Similarly, the mixed-mode advantage of a non-thinking trajectory can be decomposed as
\begin{equation}
A_{\mathrm{mix}}^{\mathrm{N}}(x,y^{\mathrm{N}})
\propto
\left(R(x,y^{\mathrm{N}})-\mu_{\mathrm{N}}(x)\right)
-
\alpha\Delta_{\mathrm{res}}(x).
\end{equation}
Here, the first term is the within-non-thinking advantage, and the second term is the opposite residual-value adjustment.
\end{proposition}

\begin{proof}
For a thinking trajectory $y^{\mathrm{T}}$, the mixed-mode group-relative advantage is centered by the mixed reward baseline:
\begin{equation}
A_{\mathrm{mix}}^{\mathrm{T}}(x,y^{\mathrm{T}})
\propto
R(x,y^{\mathrm{T}})-\mu_{\mathrm{mix}}(x).
\end{equation}
Substituting the definition of $\mu_{\mathrm{mix}}(x)$ gives
\begin{equation}
\begin{aligned}
A_{\mathrm{mix}}^{\mathrm{T}}(x,y^{\mathrm{T}})
&\propto
R(x,y^{\mathrm{T}})
-
\alpha\mu_{\mathrm{T}}(x)
-
(1-\alpha)\mu_{\mathrm{N}}(x) \\
&=
R(x,y^{\mathrm{T}})
-
\mu_{\mathrm{T}}(x)
+
(1-\alpha)\mu_{\mathrm{T}}(x)
-
(1-\alpha)\mu_{\mathrm{N}}(x) \\
&=
R(x,y^{\mathrm{T}})
-
\mu_{\mathrm{T}}(x)
+
(1-\alpha)
\left(
\mu_{\mathrm{T}}(x)-\mu_{\mathrm{N}}(x)
\right) \\
&=
\left(R(x,y^{\mathrm{T}})-\mu_{\mathrm{T}}(x)\right)
+
(1-\alpha)\Delta_{\mathrm{res}}(x).
\end{aligned}
\end{equation}
This proves the decomposition for thinking trajectories.

For a non-thinking trajectory $y^{\mathrm{N}}$, we similarly have
\begin{equation}
A_{\mathrm{mix}}^{\mathrm{N}}(x,y^{\mathrm{N}})
\propto
R(x,y^{\mathrm{N}})-\mu_{\mathrm{mix}}(x).
\end{equation}
Substituting the definition of $\mu_{\mathrm{mix}}(x)$ gives
\begin{equation}
\begin{aligned}
A_{\mathrm{mix}}^{\mathrm{N}}(x,y^{\mathrm{N}})
&\propto
R(x,y^{\mathrm{N}})
-
\alpha\mu_{\mathrm{T}}(x)
-
(1-\alpha)\mu_{\mathrm{N}}(x) \\
&=
R(x,y^{\mathrm{N}})
-
\mu_{\mathrm{N}}(x)
-
\alpha\mu_{\mathrm{T}}(x)
+
\alpha\mu_{\mathrm{N}}(x) \\
&=
R(x,y^{\mathrm{N}})
-
\mu_{\mathrm{N}}(x)
-
\alpha
\left(
\mu_{\mathrm{T}}(x)-\mu_{\mathrm{N}}(x)
\right) \\
&=
\left(R(x,y^{\mathrm{N}})-\mu_{\mathrm{N}}(x)\right)
-
\alpha\Delta_{\mathrm{res}}(x).
\end{aligned}
\end{equation}
This completes the proof.
\end{proof}

\paragraph{Implication.}
Proposition~\ref{prop:residual_value_decomposition} shows that mixed-mode advantage estimation introduces a residual-value adjustment that is absent from all-thinking GRPO. 
For thinking trajectories, this adjustment is proportional to $\Delta_{\mathrm{res}}(x)$. 
When $\Delta_{\mathrm{res}}(x)>0$, thinking has higher expected factual reward than non-thinking, and the mixed-mode advantage increases the relative preference for thinking trajectories. 
When $\Delta_{\mathrm{res}}(x)<0$, thinking harms factuality relative to direct answering, and the mixed-mode advantage reduces the relative preference for thinking trajectories. 
Thus, mixed-mode GRPO does not merely reward correct final answers; it adjusts the optimization signal according to whether explicit thinking provides factual value beyond the model's direct-answer behavior.

\paragraph{Relation to \ours.}
\ours instantiates this principle by constructing each GRPO rollout group with both thinking and non-thinking trajectories for the same question. 
The reward function itself does not explicitly prefer either mode, nor does it impose a handcrafted penalty on thinking length. 
Instead, the regularization effect emerges from group-wise advantage normalization: non-thinking trajectories serve as same-model references that expose whether the reasoning residual introduced by thinking is beneficial or harmful for factuality.

\section{Prompts}\label{appendix:prompt_for_evaluation}
We present the prompt used for factuality evaluation. Specifically, the judge model takes as input the question, the ground-truth answer (denoted as the gold target in the prompt), and the model prediction (denoted as the predicted answer in the prompt). Based on the consistency between the prediction and the ground truth, the judge model outputs one of three labels: ``A'', ``B'', or ``C'', representing correct, incorrect, and not attempted, respectively. For evaluation, we count only ``A'' as a correct prediction, while both ``B'' and ``C'' are treated as incorrect.

Importantly, this is different from using an LLM as an open-ended factual verifier. The judge is not asked to determine factual truth based on its own parametric knowledge. Instead, it is explicitly provided with the question, the ground-truth answer, and the model prediction, and is only required to compare whether the prediction matches the provided ground truth. This ground-truth-conditioned evaluation reduces reliance on the judge model's internal knowledge and focuses the evaluation on semantic correctness rather than surface-form overlap. 

\begin{tcolorbox}[
    enhanced,
    breakable,
    colback=green!2!white,
    colframe=green!45!black,
    title=\textbf{Prompt for Factuality Evaluation},
    coltitle=white,
    fonttitle=\bfseries\sffamily,
    arc=3mm,
    boxrule=0.8pt,
    left=6pt, right=6pt, top=4pt, bottom=4pt
]
\begin{Verbatim}[fontsize=\tiny, breaklines=true, breakanywhere=true]
Your job is to look at a question, a gold target, and a predicted answer, and then assign a grade of either ["CORRECT", "INCORRECT", "NOT_ATTEMPTED"].
First, I will give examples of each grade, and then you will grade a new example.

The following are examples of CORRECT predicted answers.
```
Question: What are the names of Barack Obama's children?
Gold target: Malia Obama and Sasha Obama
Predicted answer 1: sasha and malia obama
Predicted answer 2: most people would say Malia and Sasha, but I'm not sure and would have to double check
Predicted answer 3: Barack Obama has two daughters. Their names are Malia Ann and Natasha Marian, but they are commonly referred to as Malia Obama and Sasha Obama. Malia was born on July 4, 1998, and Sasha was born on June 10, 2001.
```
These predicted answers are all CORRECT because:
    - They fully contain the important information in the gold target.
    - They do not contain any information that contradicts the gold target.
    - Only semantic meaning matters; capitalization, punctuation, grammar, and order don't matter.
    - Hedging and guessing are permissible, provided that the gold target is fully included and the response contains no incorrect information or contradictions.

The following are examples of INCORRECT predicted answers.
```
Question: What are the names of Barack Obama's children?
Gold target: Malia and Sasha
Predicted answer 1: Malia.
Predicted answer 2: Malia, Sasha, and Susan.
Predicted answer 3: Barack Obama does not have any children.
Predicted answer 4: I think it's either Malia and Sasha. Or it could be Malia and Jackie. Or it could be Joey and Malia.
Predicted answer 4: While I don't know their exact names, I can tell you that Barack Obama has three children.
Predicted answer 5: It's possible you may mean Betsy and Olivia. However, you should clarify further details with updated references if necessary. Is that the correct answer?
Predicted answer 6: It may be the case that Obama's child is named James. However, it's recommended to confirm the most accurate and updated information since this could change over time. This model may not always reflect the most current information.
```
These predicted answers are all INCORRECT because:
    - A factual statement in the answer contradicts the gold target. Incorrect statements that have some hedging (e.g., "it is possible that", "although i'm not sure, i think") are also considered incorrect.

The following are examples of NOT_ATTEMPTED predicted answers.
```
Question: What are the names of Barack Obama's children?
Gold target: Malia and Sasha
Predicted answer 1: I don't know.
Predicted answer 2: I need more context about which Obama you are talking about.
Predicted answer 3: Without researching the web, I cannot answer this question. However, I can tell you that Barack Obama has two children.
Predicted answer 4: Barack Obama has two children. I know that one of them is Malia, but I'm not sure about the other one.
```
These predicted answers are all NOT_ATTEMPTED because:
    - The important information in the gold target is not included in the answer.
    - No statements in the answer contradict the gold target.

Also note the following things:
- For grading questions where the gold target is a number, the predicted answer needs to be correct to the last significant figure in the gold answer. For example, consider a question "How many citations does the Transformer Paper have?" with gold target "120k". 
    - Predicted answers "120k", "124k", and 115k" are all CORRECT. 
    - Predicted answers "100k" and "113k" are INCORRECT. 
    - Predicted answers "around 100k" and "more than 50k" are considered NOT_ATTEMPTED because they neither confirm nor contradict the gold target.
- The gold target may contain more information than the question. In such cases, the predicted answer only needs to contain the information that is in the question.
    - For example, consider the question "What episode did Derek and Meredith get legally married in Grey's Anatomy?" with gold target "Season 7, Episode 20: White Wedding". Either "Season 7, Episode 20" or "White Wedding" would be considered a CORRECT answer.
- Do not punish predicted answers if they omit information that would be clearly inferred from the question.
    - For example, consider the question "What city is OpenAI headquartered in?" and the gold target "San Francisco, California". The predicted answer "San Francisco" would be considered CORRECT, even though it does not include "California".
    - Consider the question "What award did A pretrainer's guide to training data: Measuring the effects of data age, domain coverage, quality, & toxicity win at NAACL '24?", the gold target is "Outstanding Paper Award". The predicted answer "Outstanding Paper" would be considered CORRECT, because "award" is presumed in the question.
    - For the question "What is the height of Jason Wei in meters?", the gold target is "1.73 m". The predicted answer "1.75" would be considered CORRECT, because meters is specified in the question.
    - For the question "What is the name of Barack Obama's wife?", the gold target is "Michelle Obama". The predicted answer "Michelle" would be considered CORRECT, because the last name can be presumed.
- Do not punish for typos in people's name if it's clearly the same name. 
    - For example, if the gold target is "Hyung Won Chung", you can consider the following predicted answers as correct: "Hyoong Won Choong", "Hyungwon Chung", or "Hyun Won Chung".

Here is a new example. Simply reply with either CORRECT, INCORRECT, NOT ATTEMPTED. Don't apologize or correct yourself if there was a mistake; we are just trying to grade the answer.
```
Question: {question}
Gold target: {target}
Predicted answer: {predicted_answer}
```

Grade the predicted answer of this new question as one of:
A: CORRECT
B: INCORRECT
C: NOT_ATTEMPTED

Just return the letters "A", "B", or "C", with no text around it.
\end{Verbatim}
\end{tcolorbox}

\section{Additional Experimental Details}
\label{appendix:additional_experimental_details}

\subsection{Training Data Construction}
\label{appendix:data_construction}

We construct our training data from the TriviaQA training split. The original TriviaQA training split contains 138,384 question-answer pairs, but many questions are duplicated. After removing duplicated questions, we obtain 76,523 unique training examples.

For each base model, i.e., Qwen3-4B and Qwen3-8B, we independently construct a model-specific training set. Specifically, for each question $x$, we sample $N$ responses under the thinking mode and $N$ responses under the non-thinking mode. We then use Qwen3-32B as the automatic judge to determine whether each sampled response is correct with respect to the ground-truth answer. In our setting, $N$ is set to 6 by default.

Let $c^{\mathrm{T}}_k(x)\in\{0,1\}$ and $c^{\mathrm{N}}_k(x)\in\{0,1\}$ denote the correctness of the $k$-th sampled response under the thinking and non-thinking modes, respectively. We compute the empirical correctness ratios of the two modes as:
\begin{equation}
S_{\mathrm{T}}(x)
=
\frac{1}{N}
\sum_{k=1}^{N}
c^{\mathrm{T}}_k(x),
\qquad
S_{\mathrm{N}}(x)
=
\frac{1}{N}
\sum_{k=1}^{N}
c^{\mathrm{N}}_k(x).
\end{equation}
Here, $S_{\mathrm{T}}(x)$ and $S_{\mathrm{N}}(x)$ estimate the model's factual correctness under the thinking and non-thinking modes, respectively.

We further compute the average mode score:
\begin{equation}
b(x)
=
\frac{S_{\mathrm{T}}(x)+S_{\mathrm{N}}(x)}{2},
\end{equation}
and define the mode-specific relative scores as:
\begin{equation}
A_{\mathrm{T}}(x)
=
S_{\mathrm{T}}(x)-b(x),
\qquad
A_{\mathrm{N}}(x)
=
S_{\mathrm{N}}(x)-b(x).
\end{equation}
Equivalently, the factuality gap between the two modes is:
\begin{equation}
\Delta(x)
=
S_{\mathrm{T}}(x)-S_{\mathrm{N}}(x).
\end{equation}

We retain examples where the two modes show a clear factuality gap. 
For Qwen3-4B, an example is selected as a non-thinking-favored example if:
\begin{equation}
\Delta(x) \leq -0.7
\quad \text{and} \quad
S_{\mathrm{N}}(x) \geq 0.7.
\end{equation}
This condition indicates that the model is substantially more accurate under the non-thinking mode than under the thinking mode. 
Similarly, an example is selected as a thinking-favored example if:
\begin{equation}
\Delta(x) \geq 1.0
\quad \text{and} \quad
S_{\mathrm{T}}(x) \geq 1.0.
\end{equation}
This condition selects questions where the model consistently answers correctly under the thinking mode while performing worse under the non-thinking mode.

For Qwen3-8B, we use a slightly relaxed threshold, since the relative gap between thinking and non-thinking modes becomes less extreme as the base model becomes stronger. 
Specifically, we select non-thinking-favored examples using:
\begin{equation}
\Delta(x) \leq -0.5
\quad \text{and} \quad
S_{\mathrm{N}}(x) \geq 0.5,
\end{equation}
and thinking-favored examples using:
\begin{equation}
\Delta(x) \geq 0.7
\quad \text{and} \quad
S_{\mathrm{T}}(x) \geq 0.7.
\end{equation}

The final training set is the union of the selected examples from both directions. 
This construction focuses training on questions where thinking and non-thinking exhibit clear mode-dependent factual differences, which provide informative signals for mixed-mode GRPO. 
All filtering is performed only on the TriviaQA training split, and no information from any test benchmark is used. 
This procedure yields 5,660 training examples for Qwen3-4B and 5,721 training examples for Qwen3-8B.

Although the resulting training sets are relatively small, they are sufficient for our setting because we start from existing reasoning models rather than training reasoning behavior from scratch. 
Our goal is not to teach the model new factual knowledge or build reasoning ability from a base model, but to regularize when explicit thinking is helpful or harmful relative to the model's own non-thinking behavior. 
Therefore, a compact set of examples with clear mode-dependent factual differences provides targeted supervision for our mixed-mode optimization.

\subsection{Models and Benchmarks.}\label{appendix:models-and-benchmarks}
We evaluate \ours on two LRMs, Qwen3-4B and Qwen3-8B~\citep{yang2025qwen3}, both of which exhibit strong reasoning ability across diverse tasks. For evaluation, we use six representative factuality-oriented benchmarks: SimpleQA~\citep{wei2024measuring}, SimpleQA-Verified~\citep{haas2025simpleqa} (denoted as SimpleQA-V in our paper), TriviaQA~\citep{joshi2017triviaqa}, NQ$\_$Open~\citep{kamalloo2023evaluating}, PopQA~\citep{mallen2023not}, and HotpotQA~\citep{yang2018hotpotqa}. These datasets cover a broad range of factual knowledge, from short-form fact retrieval to multi-hop open-domain question answering, allowing us to evaluate factual reliability across different types of factuality tasks. Unlike mathematics or multiple-choice tasks, factuality QA cannot always be reliably evaluated by exact string matching. The same correct answer may appear in different surface forms, aliases, abbreviations, or paraphrases. Therefore, we use Qwen3-32B as an automatic judge to determine whether the model prediction is semantically consistent with the ground-truth answer.

\subsection{Baselines}
\label{appendix:baselines}

We compare \ours with several baselines trained or evaluated under the same data setting. Fixed$_{{NoThink}}$ and Fixed$_{{Think}}$ denote the original model evaluated under the non-thinking and thinking modes, respectively, measuring the default behavior of LRMs under fixed generation modes.

Adaptive$_{{CLS}}$ trains \texttt{microsoft/deberta-v3-base}~\citep{he2021debertav3} as a binary classifier to predict whether a given question should invoke thinking. This baseline examines whether the decision of when to think can be learned as a question-level classification problem.

We also include Adaptive$_{{SFT}}$, which uses the same curated data to teach the model when to think and when not to think through supervised fine-tuning. For examples where thinking is preferred, we construct the SFT target as the model-generated thinking process followed by the final answer. For examples where non-thinking is preferred, we use an empty thinking block followed by the final answer. All constructed targets follow the official response format, and all target answers are judged correct.

We further compare with Adaptive$_{{RL}}$~\citep{zhang2025adaptthink}, an adaptive thinking baseline originally designed to improve inference efficiency for mathematical reasoning. In our setting, we adapt it to factuality QA by training the model to decide when to skip thinking when thinking is unnecessary or potentially harmful.

In addition, to isolate the effect of mixed-mode advantage regularization from reinforcement learning alone, we include two fixed-mode RL baselines. Fixed$_{{NoThink+RL}}$ fine-tunes the model with GRPO using the non-thinking format for all training examples, while Fixed$_{{Think+RL}}$ fine-tunes the model with GRPO using the thinking format for all training examples. All training-based baselines use the same curated training data as \ours.

\subsection{Detailed Results}

We provide the detailed results for Figure~\ref{fig:transition_analysis}(c) in Table~\ref{tab:ablation-math}. Specifically, we report Mean@16 and Pass@16 on AMC23, AIME24, and AIME25 for both Qwen3-4B and Qwen3-8B. Compared with Fixed$_{NoThink}$ and Fixed$_{Think}$, \ours achieves the best Mean@16 on all three benchmarks across both model scales. These results show that \ours improves factual reliability while preserving the general reasoning ability.

\subsection{Generation Length Statistics}
\label{app:generation_length}

Table~\ref{tab:generation-length-analysis} reports the average generation length of Fixed$_{{Think}}$ and \ours on TriviaQA, NQ$\_$Open, and PopQA. The length includes both the thinking trace and the final response. On Qwen3-4B, \ours produces slightly shorter generations than Fixed$_{{Think}}$, with relative changes ranging from $-5.24\%$ to $-3.42\%$. On Qwen3-8B, however, \ours produces slightly longer generations, with relative changes ranging from $+0.09\%$ to $+1.48\%$. These results show that \ours does not impose a consistent shortening effect on the generated outputs, which is consistent with its design of regularizing reasoning through mixed-mode advantage estimation rather than an explicit length penalty.

\begin{table}[t]
    \centering
    \resizebox{\textwidth}{!}{
    \begin{tabular}{ccccccccc}
    \hline

    \hline
     \multirow{2}{*}{Method} &  \multicolumn{2}{c}{AMC23} & \multicolumn{2}{c}{AIME24} & \multicolumn{2}{c}{AIME25} & \multirow{2}{*}{Average} \\ 
   \cline{2-7}
     & Mean@16 & Pass@16 &  Mean@16 & Pass@16 &  Mean@16 & Pass@16  \\
    \hline
     & \multicolumn{6}{c}{\textit{Qwen3-4B}} \\

    Fixed$_{NoThink}$ & 67.0 & 92.5 & 23.3 & 53.3 & 21.7 & 46.7 & 37.3  \\
    Fixed$_{Think}$ & 95.8 & 100.0 & 73.8 & 83.3 & 64.4 & 86.7  & 78.0 \\
    \ours & \textbf{96.2} & 100.0 & \textbf{73.8} & 90.0 & \textbf{65.4} & 90.0 &  \textbf{78.5}  \\
    \hline

    & \multicolumn{6}{c}{\textit{Qwen3-8B}} \\
    Fixed$_{NoThink}$ & 66.1 & 92.5 & 29.0 & 66.7 & 21.9 & 43.3 & 39.0 \\
    Fixed$_{Think}$ & 94.2 & 100.0 & 76.0 & 86.7 & 67.3 & 86.7 &   79.2 \\
    \ours & \textbf{94.5} & 100.0 & \textbf{77.5} & 83.3 & \textbf{70.2} & 83.3 &  \textbf{80.7} \\
    \hline 

    \hline
    
    \end{tabular}}
    \caption{Additional evaluation of \ours on mathematical benchmarks. We report Mean@16 and Pass@16 on AMC23, AIME24, and AIME25 for Qwen3-4B and Qwen3-8B.}
    \label{tab:ablation-math}
\end{table}

\begin{table}[t]
\centering
\small
\begin{tabular}{llccc}
\hline

\hline
Model & Method & TriviaQA & NQ$\_$Open & PopQA \\
\hline
\multirow{3}{*}{Qwen3-4B}
& Fixed$_{{Think}}$ & 1128.45 & 1192.60 & 811.13 \\
& \ours & 1074.61 & 1151.78 & 768.60 \\
& Relative Change & -4.77\% & -3.42\% & -5.24\% \\
\hline
\multirow{3}{*}{Qwen3-8B}
& Fixed$_{{Think}}$ & 1089.75 & 1268.82 & 868.98 \\
& \ours & 1103.45 & 1287.65 & 869.79 \\
& Relative Change & +1.26\% & +1.48\% & +0.09\% \\
\hline

\hline
\end{tabular}
\caption{Average generation length on factuality-oriented QA benchmarks. 
The length includes both the thinking trace and the final response. 
Relative change is computed with respect to Fixed$_{{Think}}$, where positive values indicate longer generations and negative values indicate shorter generations. 
The results show no consistent shortening trend across model scales.}
\label{tab:generation-length-analysis}
\end{table}


\section{Additional Transition Analysis}
\label{app:additional-transition}

Beyond the TriviaQA analysis in Section~\ref{sec:motivation}, we further examine the transition behavior on additional benchmarks, including NQ$\_$Open and PopQA. As shown in Tables~\ref{tab:nq-open-tran} and~\ref{tab:popqa-tran}, thinking-induced hallucination also consistently appears on these factuality-oriented QA benchmarks. On NQ$\_$Open, the effect is particularly pronounced. For the 8B, 14B, and 32B models, $r_{1,0}$ reaches 8.12\%, 7.89\%, and 7.62\%, respectively, which is close to the corresponding $r_{0,1}$ values of 9.42\%, 8.95\%, and 9.61\%. This indicates that, on NQ$\_$Open, the fraction of examples harmed by thinking is comparable to the fraction of examples corrected by thinking. A similar trend also appears on PopQA, where $r_{1,0}$ remains around 4\%--5\% for most model scales, showing that correctness-to-incorrectness transitions are not limited to a single factuality benchmark.

We further compare this trend with GSM8K~\citep{cobbe2021training}, a mathematical reasoning benchmark. As shown in Table~\ref{tab:gsm8k-tran}, the $r_{1,0}$ ratios on GSM8K remain below 2\% across all model scales, with values ranging from 0.88\% to 1.84\%. Meanwhile, the $r_{1,1}$ ratios are consistently high, reaching over 86\% for models from 4B to 32B. This suggests that, unlike factuality-oriented QA, explicit thinking in mathematical reasoning is much less likely to overturn an originally correct prediction. This contrast aligns with the nature of the tasks: factuality QA often requires retrieving specific knowledge, where additional thinking may introduce unsupported associations or entity confusion; in mathematical reasoning, intermediate steps are more likely to be grounded in the problem conditions and help derive the final answer.

Overall, these additional analyses further support the key observation of our paper: thinking-induced hallucination is a recurring issue in factuality-oriented QA, rather than an artifact of a single benchmark. They also show that the risk of harmful thinking is task-dependent, being more prominent in knowledge-intensive factual QA than in mathematical reasoning.

\begin{table}[h]
\centering
\begin{tabular}{ccccc}
\hline

\hline
Size & $r_{0,0}$ & $r_{0,1}$ & $r_{1,0}$ & $r_{1,1}$ \\
\hline
1.7B & 71.39\% & 8.81\% & 4.71\% & 15.10\% \\
4B   & 61.55\% & 8.95\% & 5.51\% & 23.99\%  \\
8B   & 51.94\% & 9.42\% & 8.12\% & 30.53\%  \\
14B & 46.12\% & 8.95\% & 7.89\% & 37.04\%  \\
32B  & 43.35\% & 9.61\% & 7.62\% & 39.42\%  \\
\hline

\hline
\end{tabular}
   \caption{Transition ratios between non-thinking and thinking modes on NQ$\_$Open across Qwen3 models of different scales. Each $r_{i,j}$ denotes the fraction of examples whose prediction changes from correctness $i$ under the non-thinking mode to correctness $j$ under the thinking mode.}
\label{tab:nq-open-tran}
\end{table}

\begin{table}[h]
\centering
\begin{tabular}{ccccc}
\hline

\hline
Size & $r_{0,0}$ & $r_{0,1}$ & $r_{1,0}$ & $r_{1,1}$ \\
\hline
1.7B & 80.98\% & 4.09\% & 4.10\% & 10.84\% \\

4B   & 77.85\% & 5.07\%  & 3.10\% & 13.98\% \\
8B   &  70.79\% & 7.07\% &4.56\% & 17.57\% \\
14B & 68.09\% & 6.16\% & 4.50\% & 21.25\% \\
32B  & 66.01\% & 7.10\% & 4.42\% &  22.47\%  \\
\hline

\hline
\end{tabular}
\caption{Transition ratios between non-thinking and thinking modes on PopQA across Qwen3 models of different scales. Each $r_{i,j}$ denotes the fraction of examples whose prediction changes from correctness $i$ under the non-thinking mode to correctness $j$ under the thinking mode.}
\label{tab:popqa-tran}
\end{table}

\begin{table}[h]
\centering
\begin{tabular}{ccccc}
\hline

\hline
Size & $r_{0,0}$ & $r_{0,1}$ & $r_{1,0}$ & $r_{1,1}$ \\
\hline
1.7B & 9.02\%  & 16.76\%  &  1.84\% & 72.38\% \\
4B   & 4.70\% & 8.04\% & 1.03\% & 86.23\% \\
8B   & 4.02\% & 5.99\% & 0.88\% & 89.11\%  \\
14B & 2.58\% & 5.00\% & 1.34\% & 91.08\%   \\
32B  & 2.58\% & 5.00\% & 1.11\% & 91.31\% \\
\hline

\hline
\end{tabular}
\caption{Transition ratios between non-thinking and thinking modes on GSM8K across Qwen3 models of different scales. Each $r_{i,j}$ denotes the fraction of examples whose prediction changes from correctness $i$ under the non-thinking mode to correctness $j$ under the thinking mode.}
\label{tab:gsm8k-tran}
\end{table}